\documentclass{article}

\usepackage{natbib}

\usepackage{microtype}
\usepackage{graphicx}
\usepackage{subfigure}
\usepackage{booktabs} 

\usepackage{hyperref}

\usepackage[accepted]{icml2024-diffxyz}

\usepackage{amsmath}
\usepackage{amssymb}
\usepackage{mathtools}
\usepackage{amsthm}
\usepackage{thmtools, thm-restate}

\usepackage[capitalize,noabbrev]{cleveref}

\theoremstyle{plain}
\newtheorem{theorem}{Theorem}[section]

\theoremstyle{definition}

\theoremstyle{remark}

\usepackage[textsize=tiny]{todonotes}

\def\eps{\varepsilon}

\def \la{\langle} \def\ra{\rangle}

\DeclareMathOperator*{\argmin}{arg\,min}

\DeclareMathOperator*{\arginf}{arg\,inf}

\DeclareMathOperator*{\LogSumExp}{LogSumExp}

\def\cC{\mathcal{C}}

\def\cL{\mathcal{L}}

\def\cP{\mathcal{P}}

\def\cR{\mathcal{R}}

\def\cU{\mathcal{U}}

\def\cX{\mathcal{X}}
\def\cY{\mathcal{Y}}

\def\sR{{\mathbb R}}

\def\rmd{\mathrm{d}}

\def\bC{{\mathbf{C}}}
\def\bK{{\mathbf{K}}}

\def\ba{{\mathbf{a}}}
\def\bb{{\mathbf{b}}}
\def\bof{{\mathbf{f}}}
\def\bg{{\mathbf{g}}}

\def\quadcost{\frac{1}{2} \lVert x-y \rVert^2}
\def\half{\frac{1}{2}}

\icmltitlerunning{Differentiable Cost-Parameterized Monge Map Estimators}

\begin{document}

\twocolumn[

\icmltitle{Differentiable Cost-Parameterized Monge Map Estimators}



\icmlsetsymbol{equal}{*}

\begin{icmlauthorlist}
\icmlauthor{Samuel Howard}{sch}
\icmlauthor{George Deligiannidis}{sch}
\icmlauthor{Patrick Rebeschini}{sch}
\icmlauthor{James Thornton}{sch,comp}

\end{icmlauthorlist}

\icmlaffiliation{comp}{Apple}
\icmlaffiliation{sch}{Department of Statistics, University of Oxford}

\icmlcorrespondingauthor{Samuel Howard}{howard@stats.ox.ac.uk}

\icmlkeywords{Machine Learning, ICML}

\vskip 0.3in
]



\printAffiliationsAndNotice{}  

\begin{abstract}
Within the field of optimal transport (OT), the choice of ground cost is crucial to ensuring that the optimality of a transport map corresponds to usefulness in real-world applications.
It is therefore desirable to use known information to tailor cost functions and hence learn OT maps which are adapted to the problem at hand.
By considering a class of neural ground costs whose Monge maps have a known form, we construct a differentiable Monge map estimator which can be optimized to be consistent with known information about an OT map.
In doing so, we simultaneously learn both an OT map estimator and a corresponding adapted cost function.
Through suitable choices of loss function, our method provides a general approach for incorporating prior information about the \textit{Monge map itself} when learning adapted OT maps and cost functions.
\end{abstract}

\section{Introduction}

\textbf{Optimal transport.}
Mapping samples between two datasets in a meaningful way is one of the most fundamental tasks across the sciences. Optimal transport (OT) provides a principled framework for such mapping problems, and has enjoyed success in many fields including throughout biology \citep{Schiebinger2019}, and extensively in machine learning  with applications in generative modelling \citep{de2021diffusion, genevay2018learning}, differentiable sorting \citep{cuturi2019differentiable}, clustering \citep{genevay2019differentiable}, resampling \citep{corenflos2021differentiable} and self supervised learning \citep{caron2020unsupervised}.

For two probability measures $\mu$ and $\nu$ on $\sR^d$, the OT problem finds the most efficient way to transport $\mu$ to $\nu$ relative to a ground cost function $ c : \sR^d \times \sR^d \rightarrow \sR$.
The choice of cost function therefore determines the notion of optimality.

\textbf{Choice of Ground Cost.} 
The squared-Euclidean distance $c(x,y) = \quadcost$ is the default and most commonly used ground cost in computational OT methods, due to its desirable theoretical properties and ease of implementation. It allows the use of the closed-form transport map from Brenier's theorem \citep{Brenier1987}, and an elegant connection to convex analysis upon which many methods are based \citep{Makkuva2020, korotinW2GN}. However, the squared-Euclidean cost can be an arbitrary choice and may not be suitable for the given problem \citep{genevay2018learning}. Moreover, it can result in erroneous mappings that do not agree with known ground-truth information \citep{somnath2023aligned}, as demonstrated in Figure \ref{fig:inverse_OT_2d}. Recently, several works have proposed methods to estimate OT maps for more general ground costs \citep{fan2023neural, asadulaev2023neural, monge-bregman-occam2023, monge-gap}, however it remains unclear how such a cost should be chosen.
\begin{figure}[t]
        \centering
        \includegraphics[width=\columnwidth]{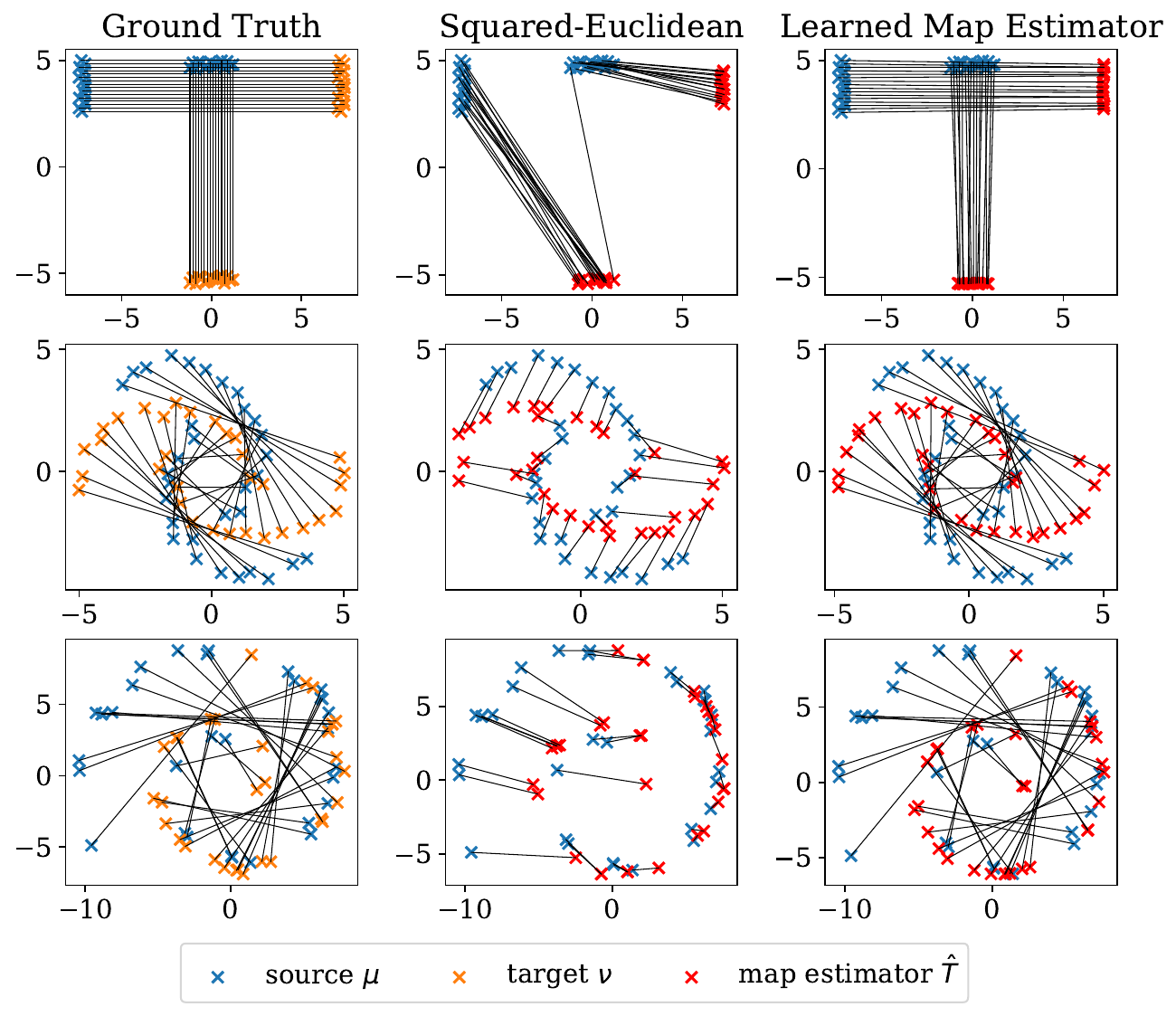}
        \vspace{-0.5cm}
        \caption{Using the squared-Euclidean cost can result in incorrect map estimators that do not align with known ground-truth mappings. By optimizing a differentiable cost-parameterized Monge map estimator to resemble known information, we can obtain Monge map estimators and corresponding cost functions which are consistent with the ground-truth.}
        \label{fig:inverse_OT_2d}
        \vspace{-0.2cm}
\end{figure}

\textbf{Cost Learning.} Ambiguity surrounding the choice of ground cost has motivated recent interest in learning adapted cost functions by leveraging assumed additional information about the transport map (see Appendix \ref{sec:cost learning} for a discussion of related work). A frequently studied setting is inverse OT, which aims to learn cost functions from paired samples assumed to come from a transport coupling.
Access to completely paired datasets is however unlikely to arise in practical applications. Instead, we may aim to learn an improved cost function from partial information about the mapping, such as a limited number of paired points or known structure in the map displacements. This motivates a need for general cost learning frameworks that are able to learn from a variety of types of available information.

One general approach for learning cost functions is a bi-level approach entailing solving a regularized OT problem for given cost function with Sinkhorn's algorithm, then differentiating through this solution to optimize the cost (see Appendix \ref{sec:sinkhorn differentiation}). In many OT applications however, one wishes to approximate the Monge map $T^\star$ itself, rather than the cost function alone. This can be used to transport out-of-sample points, enabling applications such as generative modelling and prediction \citep{CellOT}.
While a learned cost can be plugged-in to an OT map estimator, if an accurate approximation is not obtained then the resulting estimator may not agree with the information used to learn the cost.

\textbf{Contributions.}
We introduce an alternative approach for learning adapted cost functions and OT maps by instead \textit{optimizing a cost-parameterized Monge map estimator directly} to be consistent with known information. 
This approach \textit{simultaneously} learns a map estimator \textit{and} a corresponding cost function, ensuring that the resulting map has the desired properties.
We make the following contributions:
\begin{itemize}
  \setlength{\parskip}{0pt}
  \setlength{\itemsep}{0pt plus 2pt}
    \item We propose a differentiable structured Monge map estimator, incorporating costs $c(x,y) = h(x-y)$ for strictly convex $h$. We parameterize $h$ using an Input Convex Neural Network (ICNN) \citep{amos_ICNN}, $h_\theta$, and use the differentiable entropic map estimator \citep{pooladian2022entropic, monge-bregman-occam2023} to facilitate gradient based training.
    \item We then extend our construction to costs $c(x,y) = h(\Phi_\mu(x) - \Phi_\nu(y))$ for invertible functions $\Phi_{\cdot}$, providing additional flexibility whilst retaining a structured Monge map.
    \item We showcase how our approach can incorporate partially known information, including aligning mapping estimators with known data associations, and encouraging desirable properties in the resulting transport map.
\end{itemize}

\section{Background on Optimal Transport}
\textbf{Monge.} The original OT formulation given by \citet{monge1781} seeks a map $T^\star$ minimizing the total transportation cost amongst maps $T:\sR^d \rightarrow \sR^d$ pushing $\mu$ onto $\nu$:
\begin{equation} \label{eq:monge}
    \min_{T: T\#\mu = \nu} \int_{\sR^d} c(x, T(x)) \rmd \mu(x).
\end{equation}
\textbf{Kantorovich.} A more general formulation, permitting mass splitting, by \citet{kantorovich1942} seeks a joint distribution $\pi^\star \in \cP(\sR^d \times \sR^d)$ minimizing the cost over the set of couplings $\Gamma(\mu,\nu)$ between marginals $\mu$ and $\nu$: 
\begin{equation} \label{eq:kantorovich}
    \min_{\pi \in \Gamma(\mu,\nu)} \iint_{\sR^d \times \sR^d} c(x,y) \rmd \pi(x,y).
\end{equation}
The Kantorovich formulation admits the following dual formulation,
where $\mathcal{R}_c = \{(f,g) \in L^1(\mu)\times L^1(\nu) : f(x)+g(y) \leq c(x,y) ~ \mu \otimes \nu \text{~a.e.} \}$. The solutions $(f^\star,g^\star)$ to the dual problem are known as the \textit{Kantorovich potentials},
\begin{align}\label{eq:Kantorovich dual v1}
    \sup_{(f,g) \in \mathcal{R}_c} \bigg\{ \int f d\mu + \int g d\nu \bigg\}.
\end{align}

\textbf{Closed-form Monge maps.}
The following theorem details sufficient conditions for equivalence of the Monge and Kantorovich problems, and motivates solving for the Kantorovich potentials to obtain the Monge mapping.

\begin{theorem}[Theorem 1.17, \cite{santambrogio2015optimal}] \label{th:Gangbo-McCann}
    For measures $\mu$ and $\nu$ on a compact domain $\Omega \subset \sR^d$ and a cost of the form $c(x,y) = h(x-y)$ for a strictly convex function $h$, there exists an optimal plan $\pi^\star$ for the Kantorovich problem. If $\mu$ is absolutely continuous and $\partial \Omega$ is negligible, then this optimal plan is unique and of the form $(Id, T^\star)\# \mu$, there exists a Kantorovich potential $f^\star$, and
    \begin{equation}\label{eq:Gangbo-McCann}
        T^\star(x) = x - (\nabla h)^{-1}(\nabla f^\star(x)).
    \end{equation}
\end{theorem}

\textbf{Entropic OT.} The entropic OT problem instead smooths the transport plan by adding an entropic penalty term to \eqref{eq:kantorovich},
\begin{equation*}
    \pi^\star = \argmin_{\pi \in \Gamma(\mu,\nu)} \iint_{\sR^d \times \sR^d} c(x,y) \rmd \pi(x,y) + \eps KL(\pi | \mu \otimes \nu).
\end{equation*}
This relaxes the constraints on the dual potentials. The solutions $(f_\eps^\star,g_\eps^\star)$ to the dual problem are known as the \textit{entropic potentials},
\begin{align}\label{eq:entropic dual}
    \sup_{\substack{f \in \cC(\cX) \\ g \in \cC(\cY)}} &\bigg\{ \int f \rmd \mu + \int g \rmd \nu +R(f,g) \bigg\},
\end{align}
where $R(f,g):=- \eps \iint e^\frac{f(x)+g(y) - c(x,y)}{\eps} \rmd \mu(x) \rmd \nu(y)$. 

Entropic OT approaches and hence approximates OT in the limit as $\eps \searrow 0$, but the entropic solution enjoys the key benefit of differentiability with respect to the inputs and enables efficient computation for discrete measures $\hat{\mu}, \hat{\nu}$ using Sinkhorn's algorithm \citep{Cuturi2013}.

\section{A Differentiable Monge Map Estimator}
We consider learning costs of the form $c(x,y) = h_\theta(x-y)$ for a strictly convex $h_\theta$, for which the Monge map has the form in \eqref{eq:Gangbo-McCann}. This enables direct optimization according to conditions on the map itself, ensuring that the resulting mapping is consistent with known prior information.

Such costs allow the use of methods such as the entropic mapping estimator \citep{pooladian2022entropic, monge-bregman-occam2023} and $c$-rectified flow \citep{liu2022_OT_rectified}, ensuring reliable estimation of the corresponding OT map. In contrast, learning OT map estimators for arbitrary costs requires a trade-off between distribution fitting and optimality (see Appendix \ref{sec:neural OT}), and if the resulting mapping is inaccurate it may not be consistent with the information from which the cost was learned. Moreover, additional known information about an OT map will not uniquely determine a cost function. By considering convex costs, we provide an inductive bias towards simpler and more interpretable costs.

\subsection{Differentiable, Cost-Parameterized Transport Maps}
We parameterize $h_\theta$ using an Input Convex Neural Network (ICNN) \citep{amos_ICNN}. We also enforce $\alpha$-strong convexity and (optionally) symmetry of $h_\theta$ (see Appendix \ref{sec:cost function parameterization}).
We use the entropic potential $\hat{f}_\theta^\eps$ \citep{pooladian2022entropic, monge-bregman-occam2023} for the discrete empirical measures $\hat{\mu}, \hat{\nu}$ as a differentiable proxy for the Kantorovich potential $f^\star$.
This is constructed from the Sinkhorn potential $\bg_\eps^\star$ solving the discrete OT problem between $\hat{\mu}, \hat{\nu}$:
\begin{equation}\label{eq:general entropic potential}
    \hat{f}_\theta^\eps(x) = - \eps \log \sum_j \exp{\bigg(\frac{(\bg_\eps^\star)_j - h_\theta(x-y_j)}{\eps} \bigg)}.
\end{equation}
With \eqref{eq:Gangbo-McCann}, this gives the entropic mapping estimator
\begin{equation}
    T_\theta^\eps(x) = x - (\nabla h_\theta)^{-1} (\nabla \hat{f}_\theta^\eps (x)).
\end{equation}
\textbf{Differentiability.} We can differentiate through the output of Sinkhorn's algorithm using implicit differentiation \citep{luise2018differential}, or by unrolling the iterates \citep{adams2011ranking, Flamary_2018_W_discriminant}.
We also note the following well-known relation, enabling differentiation through $(\nabla h)^{-1}$.
\begin{restatable}{proposition}{gradinv}
    \vspace{-0.2cm}
    For a strictly convex function $h$, we have
    \begin{equation}\label{eq:argmin for grad_h_inv}
        (\nabla h)^{-1}(x) = \argmin_z \big\{ h(z) - \la z,x \ra \big\}.
    \end{equation}
    \vspace{-0.5cm}
\end{restatable}
As we enforce $\alpha$-strong convexity of $h$, the minimization in (\ref{eq:argmin for grad_h_inv}) can be solved efficiently using numerical methods.
We can then use implicit differentiation (see Chapter 10, \citet{blondel2024elements}) to differentiate through this inner minimization with respect to the cost parameters. We implement this using the JAXopt libary \citep{JAXopt}.

We now have a Monge map estimator that is end-to-end differentiable with respect to its cost function parameters. 

\textbf{Choice of loss function.} By choosing an appropriate loss function $\cL(\theta)$, we can incorporate desired explicit biases into the learned cost through map estimator $T_\theta^\eps$. The training procedure is described in Algorithm \ref{alg:differentiable MBO} in Appendix \ref{sec:method}.
There are many possible training losses that can be used to encourage some desired behaviour. A natural objective, for example, is to ensure the learned mapping correctly matches known paired points. Given a known subset of paired points ${(x_i,y_i)}_{i=1}^N$,  the loss $\mathcal{L}(\theta) := \frac{1}{N} \sum_{i=1}^N \lVert T_\theta^{\eps}(x_i)-y_i \rVert_2^2$ encourages the learned map to respect the known pairs. We provide other choices of loss functions suitable for a range of problems in Appendix \ref{sec:choice of loss function}.

We can also construct the reverse map estimator $\big( T_\theta^\eps(x) \big)^{-1}$. 
We find it beneficial to train using both mappings jointly; see Appendix \ref{sec:reverse maps} for details.

\textbf{Warmstarting the inner optimizations.} Each iteration of our procedure involves solving two types of inner optimization problem, (1) the discrete entropy-regularized OT problem using Sinkhorn, and (2) the evaluation of $(\nabla h)^{-1}$. While both problems are convex and thus can be solved easily, naive initializations would result in unnecessary computational cost. We instead warmstart each optimization from the corresponding solutions at the previous iteration, which significantly improves training speed \citep{amos2023_amortized_tutorial}.

\subsection{Augmenting with Diffeomorphisms} \label{sec:diffeomorphisms}
In the presence of extensive information about the map,  the choice of strictly convex costs $c(x,y)=h(x-y)$ can be overly restrictive, as there may not be a convex cost that allows the Monge map to be consistent with the known information. We therefore extend our framework by first transforming the marginal measures using diffeomorphisms $\Phi_\mu, \Phi_\nu$, then applying our method to ${\Tilde{\mu} = \Phi_\mu \# \mu}$ and ${\Tilde{\nu} = \Phi_\nu \# \nu}$. The following extension of Theorem \ref{th:Gangbo-McCann}, proved in Appendix \ref{sec:proofs}, shows that this amounts to learning a map with cost $c(x,y) = h(\Phi_\mu(x) - \Phi_\nu(y))$.
\begin{restatable}{theorem}{GenGM}
    \label{th:generalised Gangbo-McCann}
    Under the conditions of Theorem \ref{th:Gangbo-McCann} with a cost of the form $c(x,y) = h(\Phi_\mu(x) - \Phi_\nu(y))$ for a strictly convex function $h$, the optimal plan is unique and of the form $(Id,T^\star) \# \mu$, and $T^\star$ can be written as
    \begin{equation}\label{eq:generalised Gangbo-McCann}
        T^\star(x) = \Phi_\nu^{-1} \bigg[ \Phi_\mu(x) - (\nabla h)^{-1} \circ \nabla f^\star \circ \Phi_\mu (x) \bigg].
    \end{equation}
\end{restatable}
We parameterize $\Phi_\mu, \Phi_\nu$ using normalizing flows \citep{rezende_2015}, and train end-to-end along with the procedure in Algorithm \ref{alg:differentiable MBO}.  This increases the expressivity of the class of cost functions we consider. As we would prefer to learn simpler costs, we can take ${\Phi_\mu=\Phi_\nu}$ which preserves symmetry when $h$ is symmetric. \cref{{th:generalised Gangbo-McCann}} resembles similar recent results which use learnable invertible matrices \citep[Prop 3.]{klein2024elastic} and fixed mirror-maps \citep{rankin2023bregman}.

\begin{figure*}[ht!]
        \centering
        \includegraphics[width=0.9\linewidth, trim={0 .1cm 0 .1cm},clip]{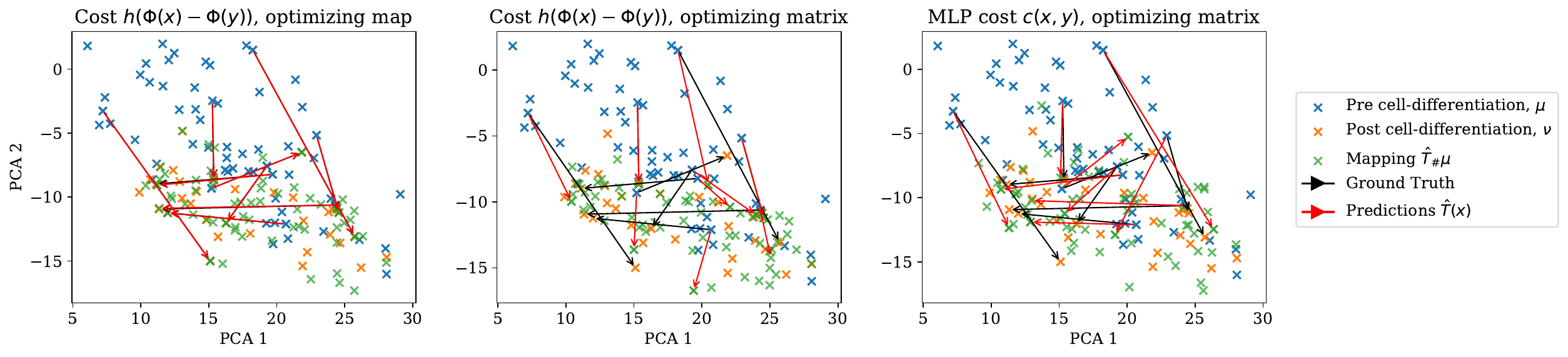}
        \vspace{-.5cm}
        \caption{\textit{(left)} By optimizing the map directly, we learn a Monge map estimator which agrees with the Live-seq trajectories. \textit{(middle)} Optimizing according to the coupling matrix with a cost 
        $h(\Phi(x)-\Phi(y))$ transports some points correctly, but struggles to learn a good cost overall (see Table \ref{tab:Live seq}). \textit{(right)} For a general MLP cost, it is difficult to obtain a good mapping estimator.}
        \label{fig:live-seq}
\end{figure*}

\section{Experiments}

\textbf{Inverse OT.} In Figure \ref{fig:inverse_OT_2d} we verify that our method can indeed learn valid cost functions in the Inverse OT setting using synthetic 2$d$ distributions.
We are able to learn a consistent map estimator for the T-shape dataset using only a symmetric convex cost $c(x,y)=h(x-y)$. For the moon and spiral datasets, we require costs of form $c(x,y) = h(\Phi_\mu(x) - \Phi_\nu(y))$. Solving the discrete OT problem on the training datasets using the learned cost recovers the correct pairs, demonstrating that the learned costs are indeed valid.

\textbf{Limited labelled pairs.} Optimal transport is commonly used to infer single-cell transcriptome trajectories between unaligned datasets, under the assumption that a cell population has moved `efficiently' between the observed timesteps \citep{Schiebinger2019, CellOT}. The recent Live-seq profiling technique \citep{Chen22_Live_seq} avoids the destruction of measured cells, enabling a number of individual cell trajectories to be traced. These trajectories appear not to agree with those predicted by OT methods that use the squared-Euclidean cost. Nevertheless, the assumption that cells move `efficiently' is reasonable, raising the question of whether a different choice of ground-cost is more suitable.

To learn an expressive but interpretable cost, we consider learning symmetric costs of the form $h(\Phi(x)-\Phi(y))$. We fit the Monge map estimator to the first 10 principal components of the Live-seq data consisting of cells pre- and post cell-differentiation, so that it agrees with the known trajectories. We plot the resulting mapping estimators in Figure \ref{fig:live-seq}. In Table \ref{tab:Live seq} we report the total incorrectly transported mass in the resulting coupling (denoted I.M.; see Appendix \ref{sec:experimental details}), as a measure of the validity of the learned cost.
To assess the adapted OT map, we also report the RMSE of the predictions for the known cell trajectories, and the Sinkhorn divergence $S_\eps(\hat{T} \# \mu, \nu)$.
Experimental details are provided in Appendix \ref{sec:experimental details}.
\begin{table}[th]
\vspace{-.5cm}
    \caption{Results from the Live-seq experiment, over 10 initialization seeds. Lower is better. Unsurprisingly, the MLP cost is able to place most mass along the correct pairings. The resulting cost is however unstructured and difficult to interpret, and it is difficult to approximate the corresponding OT map. For costs $h(\Phi(x) - \Phi(y))$, the resulting entropic mappings show an improved fit to the target, but optimizing the map learns better costs and has significantly better alignment with the ground truth.
    }
    \label{tab:Live seq}
    \begin{center}
        \begin{small}
            \begin{sc}
                \begin{tabular*}{\columnwidth}{@{}l@{\hspace{-2pt}}c@{\hspace{3pt}}c@{\hspace{3pt}}c@{}}
                    \toprule
                    & Optimizing Map & \multicolumn{2}{c}{Optimizing Matrix} \\
                    & $h(\Phi(x) - \Phi(y))$ & $h(\Phi(x) - \Phi(y))$ & MLP $c(x,y)$ \\
                    \midrule
                    I.M. & $0.095 \pm 0.035$ & $0.153 \pm 0.012$ & $0.051 \pm 0.000$ \\
                    RMSE & $2.02 \pm 2.21$ & $16.9 \pm 1.6$ & $6.91 \pm 2.52$ \\
                    $S_\eps(\hat{T} \# \mu, \nu)$ & $7.73 \pm 3.02$ & $4.94 \pm 0.85$ & $19.8 \pm 2.8$ \\
                    \bottomrule
                \end{tabular*}
            \end{sc}
        \end{small}
    \end{center}
    \vspace{-.7cm}
\end{table}
We compare against costs learned by optimizing the Sinkhorn coupling matrix to minimize the incorrect mass reported in Table \ref{tab:Live seq}. We first use the same cost parameterization $h(\Phi(x)-\Phi(y))$ and obtain a map using the entropic mapping estimator. We also use an unstructured MLP cost $c(x,y)$ and a learn a map estimator using the Monge Gap regularizer \citep{monge-gap}. Optimizing the map directly consistently obtained both an interpretable cost function and a good transport estimator agreeing with the observed trajectories.

Given the scarcity of known Live-seq trajectories, we cannot test out-of-sample performance of the resulting map estimator. We therefore perform the same procedure on synthetic data with different numbers of labelled pairs in Appendix \ref{sec:synthetic limited pairs}, reporting results on training and newly-sampled data.

\textbf{Additional Experiments.} We provide additional experiments in Appendix \ref{sec:additional experiments}, including the ability to induce desirable properties that hold on the map displacements themselves, such as being low-rank or $k$-directional.

\section{Conclusion}
We have introduced a Monge map estimator parameterized by a learnable cost function, which is differentiable with respect its parameters. This enables learning adapted cost functions and OT map estimators through gradient based training so that the map is consistent with known prior information. We have demonstrated the ability to learn costs by aligning with paired data samples, and by inducing desirable properties on the Monge map estimator itself. Future directions include: conducting further comparisons of optimizing the entropic map compared to the coupling matrix, and  investigating the use of adapted costs and OT maps in downstream tasks.

\clearpage

\section*{Acknowledgements}
Samuel Howard is supported by the EPSRC CDT in Modern Statistics and Statistical Machine Learning [grant number EP/S023151/1]. Patrick Rebeschini was funded by UK Research and Innovation (UKRI) under the UK government’s Horizon Europe funding guarantee [grant number EP/Y028333/1].



\bibliographystyle{plainnat}

\bibliography{main.bib}


\clearpage

\newpage
\appendix
\onecolumn

\section{Method} \label{sec:method}

\begin{algorithm}[h]
   \caption{Differentiable cost-parameterized entropic mapping estimator}
   \label{alg:differentiable MBO}
\begin{algorithmic}
    \STATE {\bfseries Input:} Empirical measures $\hat{\mu} = \sum_i \delta_{x_i}, \hat{\nu} = \sum_j \delta_{y_j},$
   \STATE Parameterized ICNN cost $h_\theta$ ;
   \STATE Entropy regularization scaling value $\Tilde{\eps}$;
   \STATE Loss function $\cL(\theta)$;
   \STATE Regularizing function $\cR(\theta)$;
   \WHILE{not converged}
   \STATE
   \begin{enumerate}
        \item 
        Apply Sinkhorn's algorithm between the empirical measures $\hat{\mu}, \hat{\nu}$ with entropy regularization value $\eps$ (the cost matrix mean scaled by $\Tilde{\eps}$), using the current parameterized function $h_\theta$, \;
        \item Use Sinkhorn output $\bg_\varepsilon^\star$ to construct the entropic potential estimator
        \begin{equation}
            \hat{f}_\theta^\varepsilon(x) = - \varepsilon \log \sum_j \exp{\bigg(\frac{(\bg_\varepsilon^\star)_j - h_\theta(x-y_j)}{\varepsilon} \bigg)},
        \end{equation}
        \item Construct the entropic mapping estimator
        \begin{equation}
            T_\theta^\eps(x) = x - (\nabla h_\theta)^{-1} \circ (\nabla \hat{f}_\theta^\varepsilon)(x),
        \end{equation}
        \item  Update $\theta$ by gradient descent according to a loss function $\cL(\theta) + \cR(\theta)$ by differentiating through Sinkhorn and $(\nabla h)^{-1}$ using implicit differentiation.
    \end{enumerate}
   \ENDWHILE
\end{algorithmic}
\end{algorithm}

We implement the differentiation through Sinkhorn using the OTT-JAX library \citep{ott-JAX}. We use the L-BFGS solver \citep{xiao_lbfgs} from the JAXopt library \citep{JAXopt} to evaluate and implicitly differentiate through the inner minimization when evaluating $(\nabla h_\theta)^{-1}$.

\subsection{Cost function parameterization} \label{sec:cost function parameterization}

\subsubsection{Convex function $h$}

\textbf{Input Convex Neural Networks.}
ICNNs \citep{amos_ICNN} are a class of neural networks $f_\theta : \sR^d \rightarrow \sR$ with an architecture that ensures that the mapping $x \mapsto f_\theta(x)$ is convex. An ICNN consists of $k$ feedforward layers, where for each layer $i = 0,...,k-1$ the activations are given by
\begin{equation}
    z_{i+1} = \sigma_i(W_i^{(z)}z_i + W_i^{(x)}x + b_i), \qquad f_\theta(x) = z_k.
\end{equation}
The network ensures convexity by consisting only of non-negative sums and compositions of convex non-decreasing functions with convex functions. It therefore requires the nonlinear activation functions to be convex and non-decreasing, and the weight matrices $W_i^{(z)}$ to be non-negative (no such requirements are required for the passthrough matrices $W_i^{(x)}$). The first layer does not require an additional passthrough layer, so we have $W_0^{(z)} = 0$.

ICNNs have previously been utilised in computational optimal transport methods to parameterize Kantorovich potentials, which via a reparameterization are known to be convex in the case of the squared-Euclidean cost \citep{Makkuva2020, korotinW2GN}. Although mathematically elegant, \citet{korotin2021benchmark} note that constraining the potentials to be ICNNs can in fact result in worse performance versus vanilla MLPs. We remark that in our case, the use of ICNNs instead of MLPs to parameterize the cost is crucial to our methodology, as it ensures that the inner optimization in \eqref{eq:argmin for grad_h_inv} is convex and can be solved numerically.

\textbf{Symmetry.}
The additional information used to learn an adapted Monge map and cost will not uniquely determine the cost function. This is especially the case when there is little additional information. We therefore wish to favour learning simpler and more interpretable costs. Translation-invariant convex costs $c(x,y)=h(x-y)$ provide a good inductive bias towards such costs. We may also wish for the cost to be symmetric; this can be optionally enforced by parameterizing as $h(z) = \Tilde{h}_\theta(z) + \Tilde{h}_\theta(-z)$ for an ICNN $\Tilde{h}_\theta$.

\textbf{$\alpha$-strong convexity.}
The closed-form of the Monge map in \eqref{eq:Gangbo-McCann} requires strict convexity of $h$. We enforce this by adding a quadratic term $\alpha \lVert x - y \rVert_2^2$ to the cost function. This also ensures that the inner minimization \eqref{eq:argmin for grad_h_inv} is strongly convex and thus the unique solution can be solved for efficiently with numerical methods.

\textbf{Parameterizing the convex conjugate.}
We remark that after having learned an adapted transport map, evaluating the learned mapping at a point $x$ requires a minimization problem to evaluate $(\nabla h)^{-1}$. This can be avoided by instead parameterizing the convex conjugate $h^*$ rather than $h$, and using the relation $(\nabla h)^{-1} = \nabla h^*$. While avoiding the inner minimization when evaluating the map, this instead requires an inner minimization to evaluate each entry of the cost matrix $\bC_{i,j} = c(x_i, y_j)$, which is then fed into Sinkhorn at each training iteration. The result is that significantly more inner minimizations are required during training, though with the benefit that the resulting learned map can be evaluated directly. The choice to parameterize $h$ or $h^*$ should therefore depend on whether fast evaluation of the learned mapping is preferred over training speed.

\subsubsection{Diffeomorphisms $\Phi_\mu, \Phi_\nu$}
In Section \ref{sec:diffeomorphisms}, we extend our framework to incorporate cost functions of the form $c(x,y) = h(\Phi_\mu(x) - \Phi_\nu(y))$. In our experiments, we use the MADE flow architecture \citep{MADE_flow} to parameterize $\Phi_\mu, \Phi_\nu$, implemented using the Jax-Flows library \citep{JAX-Flows}.

\textbf{Relation to metric learning.}
When taking $\Phi_\mu  = \Phi_\nu$ the cost $c(x,y) = h(\Phi(x) - \Phi(y))$ resembles a Siamese network \citep{bromley_siamese}, a common approach in the metric learning literature. Siamese networks seek to learn encoders $\Phi$ so that the resulting function $d(x,y) = \half \lVert \Phi(x) - \Phi(y)\rVert_2^2$ is an improved notion of distance for the data. However, \citet{Indyk_2017} note that such compositions can be insufficiently expressive to precisely model certain metrics. Motivated by this, \citet{Pitis2020_InductiveBias} replace the squared-Euclidean distance with a modified ICNN, providing a more expressive class of neural distances.

Our motivation is similar. Note that we could use an ICNN $\psi$ to directly parameterize a Monge map for a cost $c(x,y) = \half \lVert \Phi_\mu(x) - \Phi_\nu(y) \rVert_2^2$ according to the composition $\Phi_\nu^{-1} \circ \nabla \psi \circ \Phi_\mu$ (via Brenier's theorem and a change of variable as in Theorem \ref{th:generalised Gangbo-McCann}). However, by incorporating an ICNN we add flexibility in the cost parameterization, allowing us to favour simpler and more interpretable learned costs.

\subsection{Choice of Loss Function} \label{sec:choice of loss function}
To illustrate the generality of our method, we provide here potential choices of loss function suitable for different applications. They encourage the map estimator to be consistent with known information or desired behaviour. We remark that these suggestions are only examples of possible loss functions that can be used, and a practitioner is free to choose any differentiable loss function suitable for their needs. In particular, our method is most suitable when behaviour is desired to hold on the map estimator itself, as this is the object being optimized.

\subsubsection{Labelled datapoints}\label{sec:labelled datapoints}

\textbf{Paired datapoints.}
Consider assuming access to a subset $\{(x_i,y_i)\}_{i=1}^N$ of $\hat{\mu}, \hat{\nu}$ consisting of known pairings from the optimal mapping, so $y_i = T^\star(x_i)$. We can optimize our map to agree with the observed pairings by penalizing the $L_2$ distance between the predictions and the known targets,
\[\mathcal{L}(\theta) := \frac{1}{N} \sum_{i=1}^N \lVert T_\theta^{\eps}(x_i)-y_i \rVert_2^2.\]
In the case where all samples in $\hat{\mu}, \hat{\nu}$ are paired, this recovers the Inverse OT setting. Note that during training, we only need to evaluate $T_\theta^{\eps}$ at the paired data points.

\textbf{Subset-to-subset correspondence.}
\citet{Liu2020_OT_SI} consider a setting in which certain source subsets are known to map to certain target subsets. If our empirical source and target distributions are ${\mu = \bigcup_i \mu_i}, {\nu = \bigcup_i \nu_i}$ with $\mu_i$ known to map to $\nu_i$, then the loss function can be chosen as
\[\mathcal{L}(\theta) = \sum_i S_{\Tilde{\eps}}(T_{\theta}^{\eps} \# \hat{\mu_i}, \hat{\nu_i}),\]
where $S_{\Tilde{\eps}}$ denotes the Sinkhorn divergence \citep{genevay2018learning} with regularization parameter $\Tilde{\eps}$.

\subsubsection{Properties of Map Displacements} \label{sec:properties of map displacements}
\textbf{Low-rank displacements.} To encourage the map estimator $T_\theta^\eps$ to transport along a lower $p$-dimensional subspace, we can penalize the magnitude of the trailing singular values $\sigma_i^{(\theta)}$ of the displacement matrix $(T_\theta^\eps(x_i) - x_i)_i$, \[\cL(\theta) = \sum_{i>p} \lVert \sigma_i^{(\theta)} \rVert^2.\]

\textbf{$k$-directional displacements.}
We can encourage the displacements to lie primarily along at most $k$ distinct directions. We can parameterize $k$ directions $v^\phi = (v_1^\phi, ..., v_k^\phi)$, then maximize the cumulative smoothed minimum of the cosine distance $d_{\cos}(u,v) = 1 - \frac{u \cdot v}{\lVert u \rVert \lVert v \rVert}$, \[\cL(\theta, \phi) = - \sum_i \LogSumExp_j \Big[ - d_{\cos}\big( T_\theta^\eps(x_i) - x_i,v_j^\phi \big) \Big]. \]

\subsubsection{Reverse map estimators} \label{sec:reverse maps}
We have presented our method as optimizing according to a loss on the forward mapping estimator $\cL(\theta) = \cL(T_\theta^\eps)$. However, we can also construct the reverse entropic mapping estimator as
\begin{equation*}
    \hat{g}_\theta^\eps(y) = - \eps \log \sum_i \exp{\bigg(\frac{(\bof_\eps^\star)_i - h_\theta(x_i-y)}{\eps} \bigg)},
    \qquad
     (T_\theta^\eps)^{-1}(y) = y - (\nabla \Tilde{h}_\theta)^{-1} \circ (\nabla \hat{g}_\theta^\eps)(y),
\end{equation*}
where $\Tilde{h}_\theta(z) = h_\theta(-z)$. The reverse map estimator should also be consistent with the known information, so we optimize according to the sum of the losses of the forward and backwards mapping estimators, \[\cL(\theta) = \frac{1}{2} \cL(T_\theta^\eps) + \frac{1}{2}  \cL \big( (T_\theta^\eps)^{-1} \big).\]

\subsection{Regularization}

To ensure stable training, we add a regularization term $\cR(\theta)$ to the loss used during training.

\textbf{Cost matrix.}
During optimization, the procedure may learn a cost function which places very large absolute values on certain displacements in order to encourage the desired behaviour. Although the OT map is invariant to scaling of the cost function, if such values become extreme this can result in numerical instability. To prevent this, we add a regularization term penalizing the softmax of the largest absolute value in the cost matrix $\bC_\theta$, \[\lambda_{max} \LogSumExp_{i,j}\big( \vert c_\theta(x_i, y_j) \vert  \big)^2.\] The value of $\lambda_{max}$ can be taken to be very small, so that it has minimal effect on optimization apart from preventing extreme values.

\textbf{Flows.}
Recall that often we prefer to learn symmetric cost functions. However, when there is a large amount of available information about the transport map, it might not be possible to fit such information with a symmetric cost. In such cases, we can use two separate flows, and we can encourage the cost towards symmetricity by adding a regularizing function so that ${\Phi_\mu \approx \Phi_\nu}$, \[\sum_i \lVert \Phi_\mu(x_i) - \Phi_\nu(x_i) \rVert^2 + \sum_j \lVert \Phi_\mu(y_j) - \Phi_\nu(y_j) \rVert^2.\]

As we wish to prefer simpler cost functions, we can control the complexity of the flow by regularizing using the Dirichlet energy, \[ \half \sum_i \lVert \nabla \Phi_\mu(x_i)\rVert_2^2 + \half \sum_j \lVert \nabla \Phi_\nu(y_j)\rVert_2^2.\]

\section{Related Work} \label{sec:related work}

\subsection{Monge Map Estimation} \label{sec:Monge map estimation}
Often in applications it is desired to construct an estimator for the Monge map itself. This motivates our approach of optimizing a Monge map estimator directly, to ensure that the resulting mapping agrees with known information. We here provide an overview of methods to construct a Monge map estimator.

\subsubsection{Neural OT map estimation} \label{sec:neural OT}
Many approaches utilise neural networks to parameterize the Kantorovich potentials, or alternatively the transport map itself.

\textbf{Squared-Euclidean cost.}
In the case of the squared-Euclidean cost $c(x,y) = \quadcost$, the closed form expression for the Monge map given in Theorem \ref{th:Gangbo-McCann} reduces to the celebrated Brenier's theorem, $T^\star(x) = x - \nabla f^\star(x)$ \citep{Brenier1987}. In this case, it is common to reparameterize the potentials by letting $\psi = \frac{1}{2} \lVert x \rVert^2 - f(x)$ and $\varphi(y) = \frac{1}{2} \lVert y \rVert^2 - g(y)$. The Kantorovich dual problem (\ref{eq:Kantorovich dual v1}) becomes
\begin{align}\label{eq:Kantorovich dual v2}
    \inf_{(\psi,\varphi) \in \Tilde{\Phi}} \bigg\{ \int \psi d\mu + \int \varphi d\nu \bigg\},
\end{align}
where $\Tilde{\Phi} = \{(\psi,\varphi) \in  L^1(\mu)\times L^1(\nu) : \psi(x)+\varphi(y) \geq \langle x,y \rangle ~ \mu \otimes \nu \text{~a.e.}\}.$ There exists an optimal potential $\psi^\star$ that is convex \citep{santambrogio2015optimal}, and the Monge map is given by $T^\star(x) = \nabla \psi^\star(x)$. 

The dual objective \eqref{eq:Kantorovich dual v2} and the link to convex analysis form the basis for many successful computational techniques, though as a result they are restricted to only the squared-Euclidean cost. \citet{taghvaei2019, Makkuva2020, korotinW2GN} parameterize the dual potentials in (\ref{eq:Kantorovich dual v2}) using ICNNs. \citet{taghvaei2019} solve for the convex conjugate $\psi^*$ as an inner minimization step, and subsequent work \citep{Makkuva2020} instead parameterize both potentials as ICNNs and optimize a minimax objective. \citet{korotinW2GN} remove the minimax objective by adding a cycle-consistency regularizer to encourage the potentials to be convex-conjugate up to a constant. While the use of ICNNs to parameterize the potentials is appealing from a mathematical standpoint, results in \citet{korotin2021benchmark} suggest that it may hinder optimization.

Although the squared-Euclidean cost has desirable theoretical properties, it may not be an appropriate choice in applications. \citet{somnath2023aligned} comment that methods approximating the squared-Euclidean OT map can result in erroneous matchings when ground truth couplings are known (demonstrated in Figure \ref{fig:inverse_OT_2d}), and \citet{genevay2018learning} have shown it to have poor discriminative properties for image data. This motivates learning an improved cost function more suitable for the problem at hand.

\textbf{General Cost Functions.}
Recently, several alternative methods have been proposed that are suitable for general cost functions. \citet{fan2023neural} and \citet{asadulaev2023neural} optimize a saddle point formulation of the OT problem, directly parameterizing the transport map and a dual potential using neural networks. \citet{monge-gap} take an alternative approach, instead parameterizing the mapping with a neural network and training so that it fits to the source measure, with a \textit{Monge Gap} regularizer that encourages the map to be optimal with respect to the chosen cost. Such methods present a trade-off between fitting to the distribution, and optimality with respect to the chosen cost, both of which need to be minimized if the resulting map is to be an accurate approximation to the OT map.

\subsubsection{Entropic map estimation}

\textbf{Entropic OT.} The entropy regularized OT problem adds a entropic penalty term to the primal objective,
\begin{equation}
    \pi_\eps^\star = \arginf_{\pi \in \Gamma(\mu,\nu)} \iint_{\sR^d \times \sR^d} c(x,y) \rmd \pi(x,y) + \eps KL(\pi | \mu \otimes \nu).
\end{equation}
Note that only the support of the reference measure $\mu \otimes \nu$ affects the optimization problem \citep{CompOT_Peyre_Cuturi}. The entropy regularization `blurs' the solution, enabling differentiability of both the entropy-regularized OT cost and plan with respect to the input measures and cost function. As a result, optimal transport has enjoyed success in machine learning applications for which differentiability is required.

The dual formulation of the unregularized Kantorovich problem \eqref{eq:Kantorovich dual v1} has strict constraints on the potentials, leading to a difficult optimization problem. The entropic-regularized dual formulation relaxes these constraints, making it more amenable to optimization,
\begin{align}\label{eq:entropic dual}
    \sup_{\substack{f \in \cC(\cX) \\ g \in \cC(\cY)}} &\bigg\{ \int f \rmd \mu + \int g \rmd \nu \notag
    - \eps \iint e^\frac{f(x)+g(y) - c(x,y)}{\eps} \rmd \mu(x) \rmd \nu(y) \bigg\}.
\end{align}

\textbf{Stochastic dual optimization.} 
\citet{Genevay2016StochasticOF} optimize this dual objective with stochastic gradient descent using samples from the measures. In the case of continuous measures, they parameterize the potentials as kernel expansions in a Reproducing Kernel Hilbert Space. \citet{seguy_largescaleOT} instead propose using neural networks to parameterize the potentials, and also construct a deterministic map using a neural network to approximate the resulting barycentric projection. In the case of the squared-Euclidean cost, this barycentric projection corresponds to the Monge map.

\textbf{Sinkhorn.}
In practice, we often have access to discrete empirical approximations to underlying continuous distributions. In the discrete case with measures $\Tilde{\mu} = \sum_i a_i \delta_{x_i}, \Tilde{\nu} = \sum_j b_j \delta_{y_j}$ and cost matrix $\bC_{i,j} = c(x_i, y_j)$, the entropy-regularized OT problem can be written as
\begin{equation}
    \Pi_\eps^\star = \min_{\Pi \in U(\ba,\bb)} \la \Pi, \bC \ra - \eps H(\Pi), \qquad H(\Pi) = - \sum_{i,j} \Pi_{i,j} (\log\Pi_{i,j} - 1).
\end{equation}
The Sinkhorn algorithm \citep{Cuturi2013} provides an efficient way to compute the solution to the discrete entropy-regularized OT problem, and is also well-suited for parallelization on GPUs to solve multiple OT problems simultaneously. From an arbitary initialization $\bg^{(0)}$ and using the kernel matrix $\bK_{i,j} = \exp \big( \frac{-\bC_{i,j}}{\eps} \big)$, the Sinkhorn iterates are defined as
\begin{equation}
    \bof^{(\ell+1)} = \eps \log \ba - \eps \log \big( \bK e^{\bg^{(\ell)} / \eps} \big),
    \qquad
    \bg^{(\ell+1)} = \eps \log \bb - \eps \log \big( \bK^\top e^{\bof^{(\ell+1)} / \eps} \big).
\end{equation}

\textbf{Entropic mapping estimator.}
Recall that for a cost function $c(x,y) = h(x-y)$ for strictly convex $h$, the Monge map has form
\[T^\star(x) = x - (\nabla h)^{-1}(\nabla f^\star(x)),\]
where $f^\star$ is the Kantorovich potential. The Kantorovich potentials can be chosen to satisfy the following $h$-conjugacy property \citep{santambrogio2015optimal}, which is unfortunately a difficult property to enforce,
\begin{equation*}
    f^\star(x) = \min_y \big\{ h(x-y) - g^\star(y) \big\} \qquad
    g^\star(y) = \min_x \big\{ h(x-y) - f^\star(x) \big\}.
\end{equation*}
In the entropically regularized case, the entropic potentials $f_\eps^\star, g_\eps^\star$ instead satisfy the following relation, which is a softmin relaxation of the above $h$-conjugacy property,
\begin{equation*}
    f_\eps^\star(x) = - \eps \log \int e^\frac{g_\eps^\star(y) - h(x-y)}{\eps} d\nu(y), \qquad
    g_\eps^\star(y) = - \eps \log \int e^\frac{f_\eps^\star(x) - h(x-y)}{\eps} d\mu(x).
\end{equation*}

This suggests using the entropic potentials in place of the true Kantorovich potential. In practice, we only have access to discrete empirical measures $\hat{\mu}, \hat{\nu}$ when constructing a Monge map estimator. The entropic mapping estimator therefore solves the discrete entropic OT problem between $\hat{\mu}, \hat{\nu}$ using Sinkhorn, and uses the output to construct the entropic potential $\hat{f}_\eps$ which is used in place of $f^*$ in the Monge map expression. The resulting entropic map estimator is
\[T_h^\eps(x) = x - (\nabla h)^{-1} (\nabla \hat{f}_\eps (x) ).\]
This estimator was proposed in \citet{pooladian2022entropic} for the squared-Euclidean cost, along with finite-sample guarantees on its performance. It was extended to general convex functions $h$ in \citet{monge-bregman-occam2023}, and \citet{klein2024elastic} demonstrate the versatility and good performance of the resulting estimator.

As it is constructed from the Sinkhorn iterations, the entropic mapping estimator can be computed efficiently in comparison to the alternative approaches outlined above. As we can differentiate through the Sinkhorn iterations, it is therefore suitable for our aim of constructing a differentiable Monge map estimator. Our additional contributions lie in the use of ICNNs to parameterize $h$, and in the observation that the evaluation of $(\nabla h)^{-1}$ can be differentiated through using implicit differentiation, enabling end-to-end differentiability of the estimator with respect to the cost function parameters.

\subsection{Cost Learning} \label{sec:cost learning}
Methods to learn an adapted cost functions have attracted attention since the problem was introduced by \cite{CuturiAvis2014}, but remain relatively unexplored in comparison to standard OT. Existing methods in the literature consider a variety of different problem settings.

\subsubsection{Inverse Optimal Transport}
Inverse optimal transport (iOT) aims to learn a cost function given paired samples $(x_i, y_i)$ from the optimal or entropic coupling.

\textbf{Optimizing a dual objective.}
Existing iOT approaches typically assume access to samples from the entropic coupling $(x_i,y_i) \sim \pi_\eps$, which form an empirical joint measure $\hat{\pi}_n = \frac{1}{n} \sum_i \delta_{(x_i,y_i)}$. They then perform maximum likelihood estimation for a parameterized cost $c_\theta$ by minimizing the negative log-likelihood given by the convex function \[\ell(\theta) = - \la c_\theta, \hat{\pi}_n \ra + \sup_{\pi \in \cU(a,b)} \big\{ \la c_\theta, \pi \ra - \eps H(\pi) \big\} .\]
The cost is usually parameterized as a linear combination of convex basis functions, and a convex regularizer $R(\theta)$ is added to encourage the parameter to be sparse \citep{Carlier2020SISTA, sparsistency2023} or low-rank \citep{Dupuy2019}. To avoid a bilevel optimization procedure, the problem is reformulated as minimizing the dual objective 
\begin{align} \label{eq:iOT dual objective}
    \mathcal{J}&(\theta,f,g) = \int c_\theta(x,y)-f(x)-g(y) d\hat{\pi}_n(x,y) + \eps \int \exp \big(\frac{f(x)+g(y)-c(x,y)}{\eps}\big) d\mu(x) d\nu(y) + \lambda R(\theta).
\end{align}
\citet{Ma2021} optimize a similar dual objective, and parameterize both the cost and the Kantorovich potentials using MLPs.
The numerical algorithm proposed in \cite{sparsistency2023} uses the known form of the optimal potential $g$ to instead optimize a semi-dual formulation of (\ref{eq:iOT dual objective}), which results in a better-conditioned optimization problem. 

Discrete inverse OT has been considered from the perspective of matching problems \citep{Li2019}, contrastive learning \citep{shi_contrastive_learning}, and economics \citep{Dupuy_2014, GalichonSalanie2021}, as well as from a Bayesian perspective in \citet{Stuart2019}. Rather than learning a specific cost matrix, \citet{Chiu2022} provide a theoretical analysis of the set of possible cost matrices.

The inverse OT setting assumes access to paired samples from an transport plan, which is an unlikely scenario in practice. It is more likely that we have partial information about data associations, such as a smaller subset of paired points. Alternatively, we may leverage known information about the transport map specific to the problem at hand as an inductive bias. The above inverse OT methods are unable to utilise such partial information.

\subsubsection{Differentiating through Sinkhorn} \label{sec:sinkhorn differentiation}
An alternative approach for learning cost functions is to optimize according a loss that is a function of the entropy-regularized coupling matrix, which can be obtained using Sinkhorn's algorithm. The cost parameter is then updated through the bilevel optimization problem by unrolling the Sinkhorn iterations, or using implicit differentiation. Differentiating through Sinkhorn provides a more flexible cost-learning approach in comparison to inverse OT, as a specific loss function can be chosen to encourage the desired behaviour in the coupling matrix.

\citet{Liu2020_OT_SI} use this procedure to learn a cost function from known subset correspondences. They use Sinkhorn's algorithm to solve for the entropic coupling matrix according to the current parameterized cost function. The loss is then the sum of the squares of the entries in the resulting matrix corresponding to incorrect mappings between the known subset assignments. In the extreme case, this recovers inverse OT and the loss function is the sum of the squares of the off-diagonal elements.

\citet{klein2024elastic} differentiate through Sinkhorn to instead learn cost functions based on structural assumptions on the displacements. They minimize a loss $\cL(\theta) = \la \pi_\theta^\eps, M(\theta) \ra$, where the matrix entries $M(\theta)_{ij} = \tau_\theta(x_i - y_j)$ are a parameterized convex regularization function evaluated on the displacements. As such, they aim to learn a cost so that the resulting coupling places low mass on values with large regularization value. In particular, they consider regularizing functions ${\tau_A^\bot(z) = \half \lVert  A^\bot z \rVert_2^2}$ promoting displacements in the span of the orthonormal matrix $A$, which they aim to learn.

As discussed in Appendix \ref{sec:Monge map estimation}, it is often the case that we wish to obtain an estimator for the Monge map itself. Learned costs can be plugged in to Monge map estimators, though it is only recently that such solvers have been developed that can handle general ground-cost functions (see Appendix \ref{sec:neural OT}). Such estimators can be difficult to interpret as they use a neural network rather than a closed-form mapping, and if they fail to learn accurate approximations to the OT map then the resulting mapping may not display the desired properties.
Moreover, it could be desired for properties to hold in the displacements of the mapping estimator itself, rather than for the displacements in the coupling matrix which are fixed according to the empirical distributions.
In contrast to the aforementioned approaches, we optimize a \textit{map estimator itself}. This avoids the need for a two-step procedure and instead optimizes the object of interest directly, ensuring that the resulting mapping displays the desired behaviour.

\subsubsection{Alternative approaches to cost learning} 

Learning improved cost functions has been considered from the supervised metric-learning perspective, aiming to learn a ground cost between histograms that agrees with labelled `similarity' coefficients between pairings \citep{CuturiAvis2014, Wang2012_supervisedEMD}. \citet{Heitz_2020} and \citet{pooladian2023lagrangian} learn a cost from observations of a density that is assumed to be evolving optimally. \cite{Heitz_2020} consider discrete measures supported on graphs, in which the cost is given by a geodesic on the graph parameterized by weights on each edge, whereas \citet{pooladian2023lagrangian} learn a Riemannian metric on the underlying continuous space. \citet{genevay2018learning, Paty2020} consider an adversarial setting in which the cost is chosen to maximize its discriminative power.

\section{Experimental details} \label{sec:experimental details}

\subsection{Inverse OT}
We train according to the Inverse OT loss function given in Appendix \ref{sec:labelled datapoints} using 128 pairs sampled from 2-dimensional T-shape, moon, and spiral distributions.
We plot the ground-truth pairings and learned predictions for 32 newly-sampled test points in Figure \ref{fig:inverse_OT_2d}, along with the squared-Euclidean entropic map estimator as a comparison. In all experiments, we parameterize $h$ to be symmetric and $0.01$-strongly convex, and use an ICNN with hidden layers of size $[32,32]$. We train for 500 iterations. For the moon and spiral datasets we use a relative epsilon value of 0.01; for the T-shape dataset we decay the relative epsilon from 0.05 to 0.002. We are able to learn a mapping estimator resembling the T-shape data using a simple symmetric cost of the form $c(x,y)=h(x-y)$. For the moon and spiral datasets, we require costs of the form $c(x,y) = h(\Phi_\mu(x) - \Phi_\nu(y))$. 

\subsection{Aligning OT maps to Live-seq trajectories}
\textbf{OT for cell profiling.} Tracing individual cell transcriptome trajectories is an important problem in biological applications, and aims to identify and predict responses of cells to biological processes or external treatments. Most profiling techniques result in the destruction of the cell. The resulting data therefore consist only of individual `snapshots' of the overall cell population without alignment. In order to infer how individual cells have developed between these snapshots, optimal transport methods can be used as an inductive bias to align the observed distributions, as it is assumed that the cell population has moved `efficiently' in the time between observations \citep{Schiebinger2019, CellOT}. Such methods typically use the squared-Euclidean cost, but it is unclear whether this is a suitable notion of distance for gene-expression data.

\textbf{Live-seq.} Recently, \citet{Chen22_Live_seq} propose the Live-seq profiling technique, which conducts genetic profiling using only a small extract of cytoplasm from the cell. This does not require the destruction of the cell and has minimal effect on its development, therefore enabling individual cell trajectories to be observed. As only cells which provide viable samples at both timesteps can be traced, the number of individual cell trajectories that are observed is far smaller than the number of observations in each cell population. \citet{Chen22_Live_seq} observe the trajectories of 12 individual cells before and after treatment, along with many more unpaired samples. It is therefore desirable to use this observed data to learn an improved cost function and corresponding OT map, which can then be used in downstream tasks such as trajectory prediction or lineage tracing. Inverse OT methods could be used to learn from only the paired samples, but they are unable to make use of the unpaired majority in the cell populations.

We restrict our attention to the adipose stem and progenitor cell (ASPC) populations before and after cell-differentiation, as these contain the majority of observed cell trajectories. Our source and target distributions therefore consist of 121 and 72 cell measurements respectively, including 9 trajectory pairings tracing the same cell. 

\textbf{Experiment details.}
We train a Monge map estimator for a symmetric cost of form $h(\Phi(x) - \Phi(y))$ to agree with these observed trajectories, using the first 10 principal components. We use a relative epsilon value of 0.01, enforce 0.01-strong convexity of $h$, and use an ICNN with hidden dimensions $[64,64,64]$. We train both the ICNNs and flows using the Adam optimizer with learning rates 3e-3 and 1e-3 respectively, and train for 1000 iterations.

For comparison, we consider optimizing the coupling matrix rather than the mapping. To ensure a fair comparison, we first consider using the same cost parameterization $h(\Phi(x) - \Phi(y))$, and we use the entropic mapping estimator to obtain a Monge map estimator. We also compare to learning a general unstructured cost $c(x,y)$ parameterized by an MLP. To obtain a map estimator for this cost, we use the Monge gap regularizer \citep{monge-gap}. When optimizing the coupling matrix, we maximize the total `correct mass' according to the known trajectories. That is, we use the loss function
\begin{equation*}
\cL(\theta) = - \langle M, \pi_\theta^\eps \rangle, \qquad 
M_{ij} = 
\begin{cases} 
1 & \text{if $x_i$ is known to map to $y_j$} \\ 
0 & \text{otherwise.} 
\end{cases}
\end{equation*}
In Figure \ref{fig:live-seq}, we compare the predicted trajectories of the resulting mapping estimators to see whether they are indeed consistent with the known ground-truth trajectories. In Table \ref{tab:Live seq} we report the RMSE of the predictions of the known points, and also the Sinkhorn divergence between the predicted points and the target distribution, averaged over 5 initializations. We also report the total incorrectly transported mass in the resulting entropy-regularized coupling matrices, defined as
\begin{equation*}
\langle M, \pi_\theta^\eps \rangle, \qquad 
M_{ij} = 
\begin{cases} 
1 & \text{if $x_i$ is known to map to $y_{j'}$ for $j \neq j'$, or $x_{i'}$ is known to map to $y_{j}$ for $i \neq i'$ }\\ 
0 & \text{otherwise.} 
\end{cases}
\end{equation*}
For a good learned cost, this value should be low as the entropic transport plan will transport a large amount of the mass for a point $x_i$ to its known endpoint $y_j$. Note that as the coupling matrix has the correct marginals, minimizing this quantity is equivalent to the objective used when optimizing according to the matrix.

\textbf{Results.} By optimizing the map directly, we are able to learn a mapping that aligns with the known pairings (Figure \ref{fig:live-seq}), which should be expected given that this was the objective used during training. The incorrectly transported mass in Table \ref{tab:Live seq} is also low, indicating that we have also learned a good cost function which places mass along the known trajectory pairings. In contrast, the maps obtained by optimizing the matrix do not align well with the ground-truth. For costs $h(\Phi(x) - \Phi(y))$ optimizing the matrix often struggles to learn a good cost, as evidenced by large amounts of incorrectly transport mass in Table \ref{tab:Live seq}. Some of the points for the cost $h(\Phi(x) - \Phi(y))$ are approximately correct, indicating that the learned cost has placed the mass correctly along some rows of the coupling matrix but has converged to a local minimum. The general MLP $c(x,y)$ learns a cost that places mass along the correct trajectories, which is unsurprising given the flexibility of such a cost parameterization. However, as the cost is unstructured it is difficult to interpret. Moreover, the unstructured cost means that it is difficult to obtain an estimator for the OT map, and consequently the map obtained using the Monge Gap regularizer does not align well with the Live-seq trajectories from which the cost was learned.

\section{Additional Experiments} \label{sec:additional experiments}

\subsection{Synthetic Limited Labelled Pairs} \label{sec:synthetic limited pairs}
\begin{figure*}[t]
        \centering
        \includegraphics[width=0.8\linewidth]{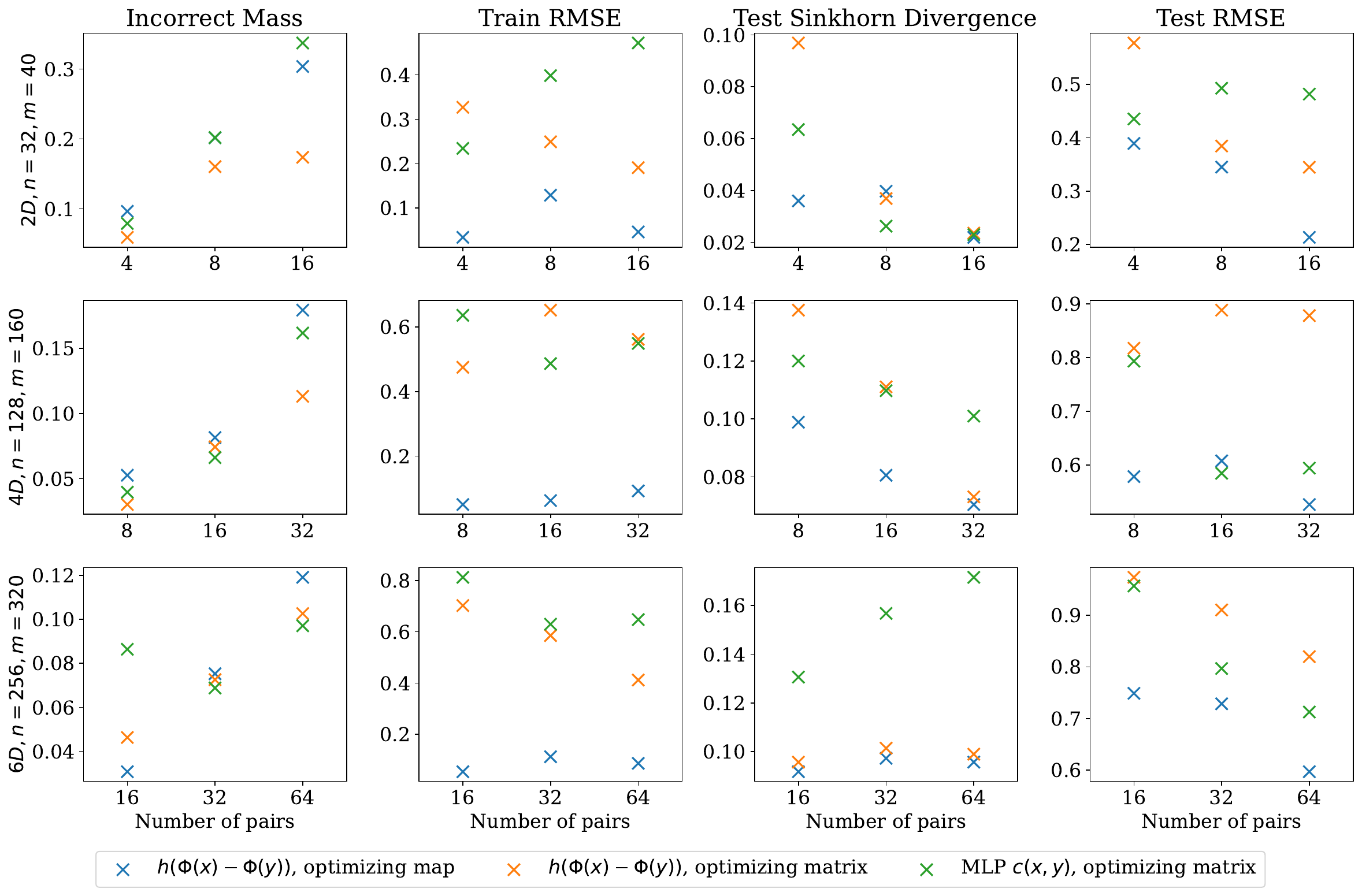}
        \caption{Results for the synthetic limited labelled pairs experiment. Optimizing the map estimator results in a mapping that aligns significantly better with the known paired points, and also demonstrates improved prediction for out-of-sample points. The learned cost appears comparable to those learned by optimizing the coupling.}
        \label{fig:synthetic limited pairs}
\end{figure*}

We validate the findings from the Live-seq experiment by performing the same procedure on similar synthetic datasets, which enables us to evaluate performance on newly sampled points. We also investigate the effect of increasing the number of paired points on the out-of-sample performance of the learned mapping.

\textbf{Data generation.} We generate the source and target datasets by pushing samples from a $\cU = \textrm{Unif}([0,1]^d)$ distribution through two respective randomly generated diffeomorphisms $\Phi_1, \Phi_2$, so we have $\mu = \Phi_{1} \# \cU, \nu = \Phi_{2} \# \cU$. This corresponds to a continuous transformation from the source to the target via the composition of the diffeomorphisms, $\nu = (\Phi_2 \circ \Phi_1^{-1}) \# \mu$. Note that while it is likely not the case that such mappings constitute an optimal transport map, such transformations provide data distributions that resemble the Live-seq data. In the experiments, the datasets consist of an unpaired majority (i.e. most are generated by independently sampling from $\cU$ and pushing through the corresponding mapping). A subset of the datasets are paired, meaning that $y = \Phi_2 \circ \Phi_1^{-1}(x)$. Such pairs simulate the known trajectories in the Live-seq data.

\textbf{Experimental details.} We perform the procedure in 2, 4 and 6 dimensions, with 32, 128 and 256 source samples and 40, 160 and 320 target samples respectively. A subset of these datasets are known pairings. We use the same experimental setup as for the Live-seq data. In Figure \ref{fig:synthetic limited pairs}, we report the incorrectly transported mass in the entropy-regularized coupling matrix for the training data as previously, as well as the RMSE error for the observed paired points. We also generate unseen test samples $\Tilde{\mu}, \Tilde{\nu}$ consisting of fully paired points, and report the Sinkhorn divergence $S_\eps(\hat{T} \# \Tilde{\mu}, \Tilde{\nu})$ and the RMSE of the resulting mapping estimator $\hat{T}$ for these newly-sampled points. The results are averaged over 5 different randomly-generated data distributions.

\textbf{Results.} The results are consistent with the observations in the Live-seq experiment. 
Optimizing the mapping appears to result in good learned costs. The amount of mass placed along incorrect directions in the coupling matrix is generally similar to those learned by optimizing the matrix, which are in fact optimized to minimize this objective.

The OT map estimator obtained from optimizing the mapping shows significantly better alignment with the known pairings, which is as expected given that is the objective being optimized. The resulting map estimator also provides better out-of-sample performance on the newly-sampled test points. In contrast, those learned from optimizing the matrix are much less consistent with the known ground-truths and generally perform worse on the out-of-sample points. Optimizing the mapping also appears to result in a mapping giving a lower Sinkhorn divergence when transporting the newly sampled points, demonstrating an improved fit to the target at a distributional level.

We also remark that the entropic mapping obtained from the cost $h(\Phi(x) - \Phi(y))$ learned by optimizing the matrix occasionally failed to give an appropriate mapping (giving very large Sinkhorn divergences between the mapped points and the target distribution). This is presumably because of a poor choice of $\eps$ when constructing the estimator. We disregard such results when calculating the averages in Table \ref{fig:synthetic limited pairs}. In contrast, optimizing the map directly ensured that the final entropic mapping was always reasonable.

\subsection{Inducing properties on the transport map displacements}
As we are optimizing according to the OT mapping estimator itself, we can also optimize to encourage properties we wish to hold on the resulting displacements themselves. This allows us to leverage knowledge about the structure of the mapping as an inductive bias when learning adapted cost functions and OT map estimators. We demonstrate the ability to induce low-rank and $2$-directional displacements by training according to the loss functions proposed in Appendix \ref{sec:properties of map displacements}, with an addition Sinkhorn divergence term $S_\epsilon(\hat{T} \# \mu, \nu)$ to ensure a good fit to the target distribution. We train the map estimators using 3-dimensional empirical measures, each consisting of 128 datapoints, and use a symmetric cost of the form $c(x,y)=h(x-y)$ with $\alpha=0.01$ and an ICNN with hidden dimensions $[64,64]$. Figure \ref{fig:combined_3d} plots the learned mappings applied to 128 newly sampled points, again with the squared-Euclidean entropic map estimator as a comparison. We see that the learned mappings display the desired structural properties.
\begin{figure*}[t]
        \centering
        \includegraphics[width=\linewidth]{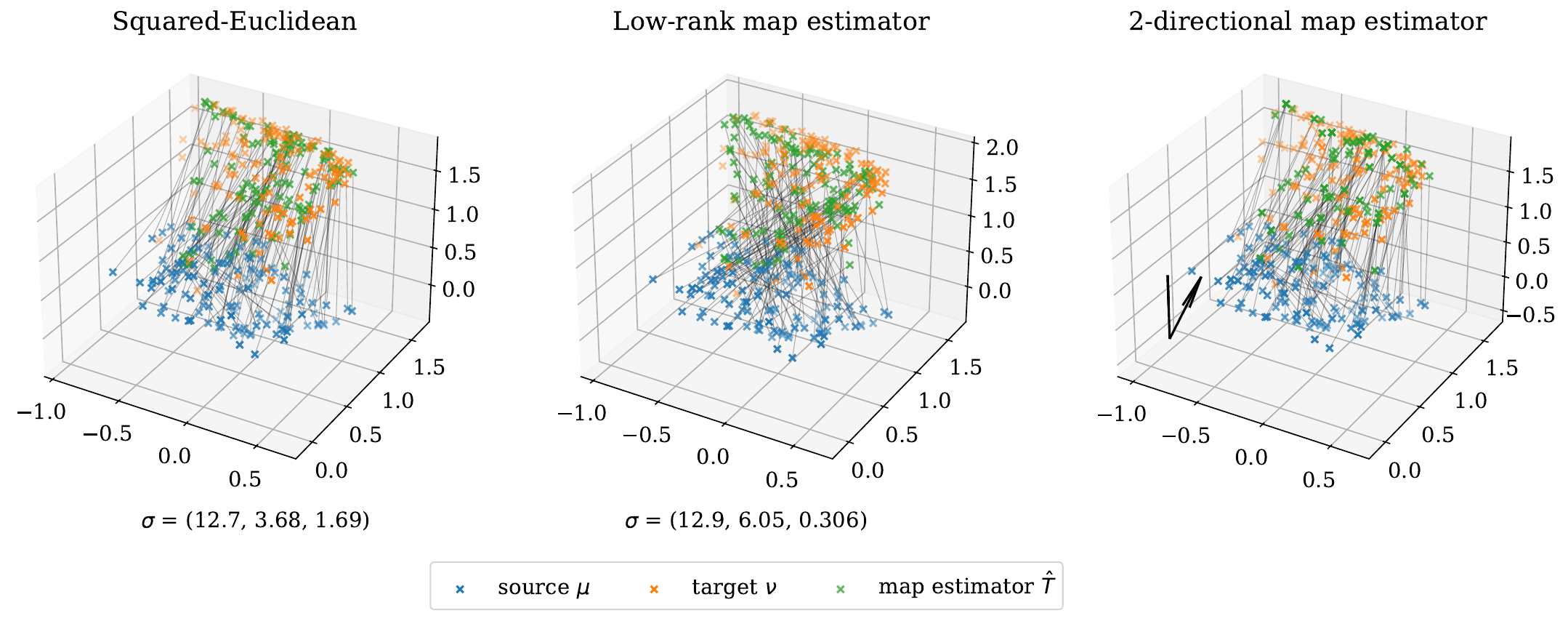}
        \caption{\textit{(middle)} The learned low-rank mapping exhibits displacements primarily along a 2-dimensional plane, as demonstrated by the small final singular value. \textit{(right)} The displacements of the learned 2-directional mapping occur primarily along the displayed directions, which were learned during training.}
        \label{fig:combined_3d}
\end{figure*}

\section{Proofs} \label{sec:proofs}

\gradinv*
\begin{proof}
    Fix $x \in \sR^d$. As $h$ is strictly convex, so is the function $g_x(z) = h(z) - \la z, x \ra$. Denote the unique minimizer of $g_x(z)$ by $z^*(x)$, which is uniquely determined by the first-order optimality condition,
    \begin{equation}
        \nabla g_x(z^*(x)) = \nabla h (z^*(x)) - x = 0.
    \end{equation}
    Rearranging and inverting $\nabla h$, we obtain $(\nabla h)^{-1}(x) = z^*(x)$ as required.
\end{proof}

\GenGM*
\begin{proof}
    Define the push-forward measures $\Tilde{\mu}=\Phi_\mu \# \mu, \Tilde{\nu}=\Phi_\nu \# \nu$ and consider the following two Kantorovich problems, denoting the original problem $(K)$ and a transformed version $(\Tilde{K})$.
    \begin{align}
        & \arginf_{\pi \in \Gamma(\mu,\nu)} \iint_{\sR^d \times \sR^d} h(\Phi_\mu(x) - \Phi_\nu(y)) \rmd \pi(x,y) \tag{$K$} \\
        & \arginf_{\Tilde{\pi} \in \Gamma(\Tilde{\mu},\Tilde{\nu})} \iint_{\sR^d \times \sR^d} h(x-y) \rmd \Tilde{\pi}(x,y) \tag{$\Tilde{K}$} 
    \end{align}

    We can construct a mapping $F: \Gamma(\mu,\nu) \rightarrow \Gamma(\Tilde{\mu},\Tilde{\nu})$ between the respective sets of admissible transport plans defined as $\pi \mapsto \Tilde{\pi} = (\Phi_\mu \otimes \Phi_\nu) \# \pi$. The fact that the coupling $\Tilde{\pi}$ has the correct marginals $\Tilde{\mu}, \Tilde{\nu}$ is a consequence of the definition of push-forward; for a test function $\varphi \in \cC^\infty(\sR^d)$, we have
    \begin{align*}
        \iint_{\sR^d \times \sR^d} \varphi(x) \rmd \Tilde{\pi}(x,y)
            & = \iint_{\sR^d \times \sR^d} \varphi(x) \rmd (\Phi_\mu \otimes \Phi_\nu) \# \pi(x,y) \\
            & = \iint_{\sR^d \times \sR^d} \varphi(\Phi_\mu(x)) \rmd \pi(x,y) \\
            & = \int_{\sR^d} \varphi(\Phi_\mu(x)) \rmd \mu(x) \\
            & = \int_{\sR^d} \varphi(x) \rmd \Tilde{\mu}(x).
    \end{align*}
    This shows that $\Tilde{\pi}$ has correct first marginal $\Tilde{\mu}$ and it can be shown similarly that the second marginal is $\Tilde{\nu}$, confirming that $\Tilde{\pi} \in \Gamma(\Tilde{\mu},\Tilde{\nu})$. Consider too the mapping $G$ given by $\Tilde{\pi} \mapsto \pi = (\Phi_\mu^{-1} \otimes \Phi_\nu^{-1}) \# \Tilde{\pi}$, which defines a map from $\Gamma(\Tilde{\mu},\Tilde{\nu})$ to $\Gamma(\mu,\nu)$. For a test function $\varphi \in \cC^\infty(\sR^d \times \sR^d)$,
    \begin{align*}
        \iint_{\sR^d \times \sR^d} \varphi(x, y) \rmd (\Phi_\mu^{-1} \otimes \Phi_\nu^{-1}) \# (\Phi_\mu \otimes \Phi_\nu) \# \pi (x,y)
            & = \iint_{\sR^d \times \sR^d} \varphi(\Phi_\mu^{-1}(x), \Phi_\nu^{-1}(y)) \rmd (\Phi_\mu \otimes \Phi_\nu) \# \pi (x,y) \\
            & = \iint_{\sR^d \times \sR^d} \varphi((\Phi_\mu \circ \Phi_\mu^{-1})(x), (\Phi_\nu \circ \Phi_\nu^{-1})(y)) \rmd \pi (x,y) \\
            & = \iint_{\sR^d \times \sR^d} \varphi(x, y) \rmd \pi (x,y).
    \end{align*}
    This shows that $G \circ F = Id$, and similarly we can show that $F \circ G = Id$. We thus conclude $G=F^{-1}$, and that $F$ is a bijection between $\Gamma(\mu,\nu)$ and $\Gamma(\Tilde{\mu},\Tilde{\nu})$.

    Define $I_K$ and $I_{\Tilde{K}}$ to be the respective infimums for the Kantorovich problems $(K)$ and $(\Tilde{K})$. For any $\pi \in \Gamma(\mu,\nu)$, the above shows that $\Tilde{\pi} = F(\pi)$ is an admissible transport plan for $(\Tilde{K})$ and thus
    \begin{align} \label{eq:lower bound on I_K}
        I_{\Tilde{K}} \leq \iint_{\sR^d \times \sR^d} h(x-y) \rmd \Tilde{\pi}(x,y)
         & = \iint_{\sR^d \times \sR^d} h(\Phi_\mu(x) - \Phi_\nu(y)) \rmd \pi(x,y).
    \end{align}
    Note that as $\Phi_\mu, \Phi_\nu$ are diffeomorphisms, the transformed Kantorovich problem $(\Tilde{K})$ satisfies the conditions of Theorem \ref{th:Gangbo-McCann} so has a unique solution $\Tilde{\pi}^\star$. Letting $\hat{\pi} = F^{-1}(\Tilde{\pi}^\star)$, we have
    \begin{align}
        \iint_{\sR^d \times \sR^d} h(\Phi_\mu(x) - \Phi_\nu(y)) \rmd \hat{\pi}(x,y)
            & = \iint_{\sR^d \times \sR^d} h(x-y) \rmd \Tilde{\pi}^\star(x,y) = I_{\Tilde{K}}.
    \end{align}
    We see that $\hat{\pi}$ attains the lower bound in (\ref{eq:lower bound on I_K}), and is therefore an optimal plan for the original Kantorovich problem $(K)$.
    
    For uniqueness, note that if there are two such optimal plans $\pi_1^\star, \pi_1^\star \in \Gamma(\mu,\nu)$, then (again using the definition of pushforwards) we have that both $F(\pi_1^\star), F(\pi_2^\star)$ attain the infimum $I_{\Tilde{K}}$. By uniqueness of $\Tilde{\pi}^\star$, we must have $F(\pi_1^\star) = F(\pi_2^\star) = \Tilde{\pi}^\star$, and inverting $F$ we see $\pi_1^\star = \pi_2^\star$ as required.

    The structure of the optimal plan $\pi^\star$ follows from that of $\Tilde{\pi}^\star$. Recall from Theorem \ref{th:Gangbo-McCann} that $\Tilde{\pi}^\star$ is of the form $(Id,\Tilde{T}^\star)$ where $\Tilde{T}^\star = Id - (\nabla h)^{-1} \circ (\nabla f^\star)$ for the Kantorovich potential $f^\star$. Applying the above result, we therefore see that the optimal plan $\pi^\star$ for our original Kantorovich problem $(K)$ is also of the form $(Id, T^\star)$. The map $T^\star$ is now given by
    \[T^\star(x) = \Phi_\nu^{-1} \bigg[ \Phi_\mu(x) - (\nabla h)^{-1} \circ \nabla f^\star \circ \Phi_\mu (x) \bigg],\]
    where $f^\star$ is the Kantorovich potential for the transformed Kantorovich problem $(\Tilde{K})$.
    
\end{proof}

\end{document}